\documentclass[%
]{mpi2015-cscpreprint}
\usepackage{geometry}
\usepackage{amsmath}
\usepackage{amsthm}
\usepackage{amssymb}
\usepackage{graphics} 
\usepackage{epsfig} 
\usepackage{amsmath} 
\usepackage{amssymb}  
\usepackage{algorithm,algorithmic}
\usepackage{cite}%
\usepackage{cancel}
\usepackage{mathtools}

\usepackage{tikz,pgfplots}

\pgfplotsset{compat=newest} 
\pgfplotsset{plot coordinates/math parser=false} 
\usepackage{color}

\usepackage{mymacros}
\newtheorem{remark}{Remark}

\newtheorem{lemma}{Lemma}

\usepackage{todonotes}

\usepackage[normalem]{ulem}

\usepackage{float}
\usepackage{color}
\usepackage{array}
\usepackage{dsfont}
\usepackage{tikz}
\usetikzlibrary{arrows,decorations.pathmorphing,backgrounds,fit,positioning,shapes.symbols,chains}
\usepackage{cleveref}
\usepackage{wrapfig}
\renewcommand{\hat}[1]{\widehat{#1}}

\newcommand{\pathfig}{./Figures}
\newcommand{\pathfigCylindereigfuc}{./Figures/Cylinder/Plots}
\newcommand{\pathfigData}{./Figures/Superposedflow/data}

\newcommand{\btQ}{\mathbf{\widetilde{Q}}}
\newcommand{\btR}{\mathbf{\widetilde{R}}}
\newcommand{\btJ}{\mathbf{\widetilde{J}}}

\newcommand{\newmethod}{\texttt{sLSI}}
\newcommand{\benchmark}{\texttt{LSI}}
\newcommand{\dt}{\texttt{dt}}

\begin{document}
\title{Inference of Continuous Linear Systems from Data with Guaranteed Stability}
\author[1]{Pawan Goyal}
\affil[1]{Max Planck Institute of Dynamics of Complex Technical Systems\authorcr
	\email{goyalp@mpi-magdeburg.mpg.de}, \orcid{0000-0003-3072-7780}}

\author[2]{Igor Pontes Duff}
\affil[2]{Max Planck Institute of Dynamics of Complex Technical Systems\authorcr
	\email{pontes@mpi-magdeburg.mpg.de}, \orcid{0000-0001-6433-6142}}

\author[3]{Peter Benner}
\affil[3]{Max Planck Institute of Dynamics of Complex Technical Systems, Otto-von-Guericke University Magdeburg \authorcr
	\email{benner@mpi-magdeburg.mpg.de}, \orcid{0000-0003-3362-4103}}

\abstract{
Machine-learning technologies for learning dynamical systems from data play an important role in engineering design. This research focuses on learning continuous linear models from data.   \emph{Stability}, a key feature of dynamic systems, is especially important in design tasks such as prediction and control. Thus, there is a need to develop methodologies that provide stability guarantees. To that end, we leverage the parameterization of stable matrices proposed in  [Gillis/Sharma, Automatica, 2017] to realize the desired models. Furthermore, to avoid the estimation of derivative information to learn continuous systems, we formulate the inference problem in an integral form. We also discuss a few extensions, including those related to control systems. Numerical experiments show that the combination of a stable matrix parameterization and an integral form of differential equations allows us to learn stable systems without requiring derivative information, which can be challenging to obtain in situations with noisy or limited data.
}

\novelty{\begin{itemize}
		\item Learning stability-guaranteed continuous linear systems.
		\item Utilizing a stable matrix parameterization.
		\item Employing an integral form of differential equations to learn continuous systems, hence does not require estimating derivative information for data.
		\item Several numerical examples demonstrate the stability-guarantee of the learned models, which otherwise yield unstable models despite good data fit.  
\end{itemize}}
\maketitle

\section{Introduction}

Linear dynamical systems (LDS) are a widely used class of dynamical systems for describing underlying physical processes in science and engineering. These models are also utilized to understand the dynamic behavior of more complex nonlinear dynamical systems around operating points or equilibria.  Consequently, LDS  can be employed for optimization, and the design of feedback control laws.
However, for complex physical processes, mathematical models are often hard to build from first-principles and physical parameters. Thus, there is a growing interest in developing mathematical models using experimental data, which is widely available due to advances in measurement technology. 
%
This work aims at learning continuous-time LDS from data while ensuring that the learned models are stable, which is necessary for goals such as long-term predictions.

Significant research has been dedicated to learning discrete-time linear dynamic systems (LDS) from time-domain data in the literature. One method that has gained widespread attention is dynamic mode decomposition (DMD), due to its accuracy in predicting dynamic systems, its connection to the Koopman operator  \cite{rowley2009spectral}, and its ease of implementation  \cite{morKutBBetal16}. DMD was originally developed in the field of fluid dynamics  \cite{morSch10}, and several variations have been proposed since, such as compressed DMD  \cite{brunton2015compressed}, extended DMD  \cite{williams2015data}, and DMD with control  \cite{morProBK16}. In the field of systems and control, methods for learning LDS from input-output data include the eigensystem realization algorithm  \cite{juang1985eigensystem}, and subspace methods \cite{van1994n4sid, viberg1995subspace}. There are also several methods for learning LDS from frequency data measurements, such as the Loewner framework  \cite{morMayA07}, vector fitting \cite{morGusS99,morDrmGB15a}, and the AAA algorithm  \cite{nakatsukasa2018aaa}. In addition, the operator inference method \cite{morPehW16} was proposed for learning linear or polynomial dynamics from state-space data in continuous-time systems. This method can be viewed as a continuous-time version of DMD, at least for linear systems. Indeed,  both involve solving a least squares problem from data. However, a key difference between both is that operator inference utilizes the derivative data, while  DMD uses the shift state data. 

Many physical processes are typically stable, meaning their state variables are well-behaved and globally bounded. These processes are often described by a set of stable differential equations, which is necessary for an accurate representation of the physics, as well as for numerical computations. However,  stability is often neglected in many frameworks for learning dynamical systems. Consequently, unstable learned dynamical systems may be inferred even if the original dynamics are stable, leading to failure in predicting long-term behavior \cite{chui1996realization}. Stability constraints based on eigenvalues can be applied to an optimization problem, but this results in a non-convex constraint optimization problem that does not scale well with complexity.
In the literature,  the identification of linear stable discrete-time systems using matrix inequality constraints has been addressed in  \cite{boots2007constraint,lacy2003subspace} in the context of subspace identification methods. In  DMD perspectives, a nonsmooth optimization method was suggested in \cite{morBenHM18} to enforce stability. Very recently,  using parametrization for discrete stable matrices \cite{gillis2019approximating}, the authors guarantee the stability of discrete linear dynamical systems, see \cite{mamakoukas2020memory}. More recently,  operator inference with matrix inequality constraints was discussed in \cite{sawant2021physics}  for a particular structured case where the linear matrix or operator is symmetric and negative definite.

In this study, we propose a framework for learning continuous-time LDS enforcing stability by construction. To achieve this, we exploit the stable matrix parameterization proposed in  \cite{gillis2017computing}, which guarantees stability for continuous-time LDS. This allows us to formulate the identification problem as an unconstrained optimization loss function. Hence,  this methodology does not require any stability constraints, e.g., by means of eigenvalues or matrix inequalities, and therefore it scales to more significant complexity problems. Additionally, we use an integral form of differential equations to avoid the need to estimate derivative information, which can be challenging to obtain for noisy or scarce data. This idea is highly inspired by the recent advances in Neural ODEs \cite{chen2018neural} and the use of integrating schemes in learning of dynamical systems \cite{gonzalez1998identification,goyal2022discovery, uy2022operator}. Through numerical experiments, we illustrate that the new methodology (\newmethod) yields systems that are guaranteed to be stable. In contrast, the classical methods---that do not enforce this property---can produce unstable models. Moreover, we show that the inference methodology is suitable for high-dimensional data by combining it with a compression step to obtain a low-dimensional data representation. 

The reminder of this paper is organized as follows. In \Cref{sec:DMD}, we provide a brief overview of a method for learning continuous-time LDS through a DMD-like procedure. The learned models are obtained through the solution of an unconstrained least-squares problem using the available data. In \Cref{sec:Stabil}, we recall the stability concept for continuous-time LDS and  present the characterization for stable matrices, proposed in \cite{gillis2017computing}, which forms the basis for our proposed methodology. We then formulate a suitable optimization problem for learning LDS, enforcing stability. We further show how to rewrite the optimization problem as an unconstraint one. In \Cref{sec:intergrationscheme}, we discuss incorporating integration schemes in learning continuous systems to avoid the necessity of derivative information, which yields a robust performance of \newmethod~under noisy and limited data conditions. In \Cref{sec:furtherconsider}, we show how the methodology can be adapted for high-dimensional data and discuss a few possible extensions. Finally, \Cref{sec:Exp} presents numerical experiments illustrating that \newmethod~guarantees to produce stable models, and \Cref{sec:Conc} provide a summary of our findings.

\section{Continuous-time Linear Systems and its Inference}\label{sec:DMD}
Throughout this work, we consider continuous-time linear dynamical systems  as follows:
\begin{equation}\label{eq:LinearDynamics}
	\frac{d}{dt}\bx(t) =  \bA\bx(t), \quad \bx(0) = \bx_0,
\end{equation}
where $\bx(t)\in\R^n$ is the state variable, $\bx_0 \in \R^{n}$ is the initial condition, and $\bA\in \R^{n \times n}$ is a matrix.  The inclusion of control variables in such systems is discussed in \Cref{sec:furtherconsider}. Our goal is to learn a stable linear dynamical system from state measurements.
The DMD algorithm finds the linear operator $\bA\in \R^{n \times n}$ that best fits the available data in the desired norm. In this work, we will focus on the analog to the DMD algorithm for continuous-time systems as \eqref{eq:LinearDynamics}. This methodology is also referred to in the literature as \emph{operator inference} with only linear term  \cite{morPehW16}. For continuous-time systems, we assume that snapshots of the state $\bx(t)$ and the state derivative $ \frac{d}{dt}\bx(t)$ are available for time instances $\{t_1, \dots, t_N\}$. These snapshots are collected in the snapshot matrices as follows:
\begin{equation}\label{ea:data}
	\bX = \begin{bmatrix} 
		| & | & & |
		\\ 
		\bx(t_0) & \bx(t_1)& \ldots & \bx(t_N) 
		\\
		| & | & & |
	\end{bmatrix} \in \R^{n \times N},
	\quad \text{and} \quad 
	\dot{\bX} = \begin{bmatrix} 
		| & | & & |
		\\ 
		\frac{d}{dt}\bx(t_0) & \frac{d}{dt}\bx(t_1)& \ldots & \frac{d}{dt}\bx(t_N) 
		\\
		| & | & & |
	\end{bmatrix} \in \R^{n \times N}.
\end{equation}
In terms of these matrices, the operator $\bA$ can be inferred via a least squares fitting. More precisely, the matrix $\bA$ is the solution to the optimization problem
\begin{equation}\label{eq:DMD_cont}
	\bA = \underset{\tilde\bA \in \R^{n \times n}}{\arg\min} \left\|\dot{\bX}-\tilde\bA\bX\right\|_F,
\end{equation}
where $\| \cdot \|_F$ is the Frobenius norm. Additionally, the minimal-norm solution for the previous optimization problem can be given as follows:
\[ \bA = \bX^{\dagger}\dot{\bX}, \]
where $\bX^{\dagger}$ denotes the pseudo-inverse of the snapshot matrix $\bX$.

It is worth noticing that the learned linear model using \cref{eq:DMD_cont} is not guaranteed to be stable, even if the available data was obtained from an underlying stable system. Another drawback of the linear regression~\eqref{eq:DMD_cont} is that it requires knowledge of the derivative of the state. Indeed, it can be difficult to accurately estimate this derivative information when dealing with noisy or limited data. To circumvent this issue, the authors e.g., in \cite{gonzalez1998identification,chen2018neural, goyal2022discovery,uy2022operator} have suggested reformulating the problem in an integral form, which has been shown to provide robust performance. In this work, we also employ a similar concept to avoid computation of the derivative information which shall be discussed more in \Cref{sec:intergrationscheme}.

\section{Stability-guaranteed Learning}\label{sec:Stabil} 
We begin by describing the characterization of stability for linear dynamical systems~\eqref{eq:LinearDynamics} which will be used on this work. It is known that a linear system \eqref{eq:LinearDynamics} is stable if and only if all of the eigenvalues of the matrix $\bA$ are in the closed left half complex plane and all eigenvalues on the imaginary axis are semi-simple. Additionally, the linear system \eqref{eq:LinearDynamics} is asymptotically stable if and only if all of the eigenvalues of the matrix $\bA$ are in the closed left half complex plane. 

\paragraph{Parametrization of stable matrices:} A recent characterization of stable matrices is discussed in \cite[Lemma~1]{gillis2017computing}. In this work, the authors show that every stable matrix $\bA$ can be written as: 
\begin{equation}\label{eq:par_stabil}
	\bA = (\bJ - \bR)\bQ  
\end{equation}
where $\bJ = -\bJ^{\top}$ is a skew symmetric matrix, $\bR = \bR^{\top} \geq 0 $ i symmetric positive semi-definite,  and $\bQ = \bQ^\top>0$ is symmetric positive definite. The authors discussed such parametrization to determine the distance to stability of any given matrix. 

It is worthwhile mentioning that the parametrization \eqref{eq:par_stabil} encodes a Lyapunov function for the underlying linear dynamical system, which is stated in the following lemma.
\begin{lemma}
Consider a linear dynamical system as in \eqref{eq:LinearDynamics} where $\bA$ takes the form given in \eqref{eq:par_stabil}. Then,  the quadratic function
\begin{equation*}
	\bV(\bx(t)) = \dfrac{1}{2}\bx(t)^{\top} \bQ\bx(t)
\end{equation*}
is a Lyapunov function of the system. 
\end{lemma}
\begin{proof}
Note that 
\begin{align*}
	\frac{d}{dt}\bV(\bx(t)) &= \dfrac{1}{2}\left(\frac{d}{dt}\bx(t)\right)^{\top}\bQ\bx(t) + \dfrac{1}{2} \bx(t)^{\top}\bQ\left(\frac{d}{dt}\bx(t)\right) \\
	&= \bx(t)^{\top}\bQ\left(\frac{d}{dt}\bx(t)\right)  
	= \bx(t)^{\top}\bQ\bA\bx(t)\\
	& = \bx(t)^{\top}\bQ(\bJ -\bR)\bQ\bx(t) \\
	& = \tilde\bx(t)^{\top}(\bJ -\bR)\tilde\bx(t) \quad \text{with} ~~\tilde\bx(t) = \bQ\bx(t) \\
	& = \tilde\bx(t)^{\top}(\bJ -\bR)\tilde\bx(t) \\
   &= \cancelto{0}{ \tilde\bx(t)^{\top}\bJ\tilde\bx(t)} -  \tilde\bx(t)^{\top}\bR\tilde\bx(t) \leq 0,
\end{align*}
because $\bR \geq 0$.
\end{proof}

\begin{remark} We also highlight that if $\bR$ is a positive definite matrix, $	\frac{d}{dt}\bV(\bx(t)) <0$. As a consequence, the matrix $\bA$ is asymptotically stable.  
\end{remark}

The parametrization in \eqref{eq:par_stabil} enables us to characterize every continuous-time stable dynamical system. Instead of relying on constraints involving matrix inequalities or eigenvalues, we will utilize this parametrization to ensure the stability of the learned dynamical system.

%
\paragraph{Stability informed learning}\label{subsec:stability_parameterization}
As noted earlier, the matrix decomposition \eqref{eq:par_stabil} guarantees that the matrix $\bA$ is stable. So,  we will use the stable matrix parametrization to guarantee that the learned linear systems are stable. To this aim, we formulate our objective function as follows:
\begin{equation}\label{eq:stable_learning1}
	\begin{aligned}
		(\bJ, \bR, \bQ) &= \underset{\tilde\bJ, \tilde\bR,\tilde \bQ}{\arg\min} \left\|\dot{\bX}-(\tilde\bJ - \tilde\bR)\tilde\bQ\bX\right\|_F,\\
		& \qquad \text{subject to} ~~ \tilde \bJ = -\tilde\bJ^\top,  \tilde\bR =  \tilde\bR^\top \geq 0,  \tilde\bQ =  \tilde\bQ^\top > 0.
	\end{aligned}
\end{equation}
Once we have the optimal value for the tuple $(\bJ,\bR,\bQ)$, we can construct the stable matrix $\bA$ as $(\bJ - \bR)\bQ$. 
Note that the optimization problem \eqref{eq:stable_learning1} includes constraints on the matrices $\btJ, \btR, \btQ$. With an appropriate parameterization for these matrices, we can re-formulate the optimization problem \eqref{eq:stable_learning1} as an unconstrained one. To that end, first, note that any skew-symmetric matrix $\btJ$ can be parameterized as
\begin{equation}\label{eq:par_ss}
	\tilde\bJ = \bar\bJ - \bar\bJ^\top,
\end{equation}
where $\bar\bJ\in \Rnn$ is a square matrix. 
Moreover, to remove the symmetric positive (semi)definite constraints, we parameterize $	\tilde\bQ$ and $	\tilde\bR$ as
\begin{equation}\label{eq:par_spd}
		\tilde\bQ = \bar\bQ\bar\bQ^\top \quad \text{and} \quad 	\tilde\bR = {\bar\bR}\bar\bR^\top, 
\end{equation} 
for any given matrices $\bar\bQ$ and $\bar\bR \in \R^{n \times n}$.  Notice that the parametrization \eqref{eq:par_spd} only guarantees that $\tilde\bQ$ is positive semidefinite unless the matrix $\bar\bQ$ is full rank. In order to avoid the rank constraint, we will relax the problem, where we allow $\tilde\bQ$ to be semidefinite. 
These parametrizations enable us to re-formulate the objective constraint function \eqref{eq:stable_learning1} as the unconstrained objective function:
\begin{equation}\label{eq:Loss_function}
	(\bar\bJ, \bar\bR, \bar\bQ)= \underset{\acute\bJ, \acute\bQ, \acute\bR}{\arg\min} \left\|\dot{\bX}-(\acute\bJ -\acute\bJ^{\top} - \acute\bR\acute\bR^\top) \acute\bQ\acute\bQ^\top\bX\right\|_F.
\end{equation}
This then results in a stable matrix $\bA = \left(\bar\bJ -\bar\bJ^{\top} - \bar\bR\bar\bR^\top\right) \bar\bQ\bar\bQ^\top$. 

\section{Un-rolling Integrating Schemes to Avoid Derivative Information Requirement }\label{sec:intergrationscheme}
One of the challenges to learning continuous-time LDS (see, e.g., \eqref{eq:Loss_function}) is the requirement for derivative information, which can be difficult to accurately  estimate  from noisy and scarce data. As an alternative, researchers have explored the use of integral forms of differential equations and numerical techniques to learn continuous systems \cite{gonzalez1998identification, goyal2022discovery,uy2022operator} .  With a similar spirit, we focus on embedding a Runge-Kutta fourth-order integration scheme to predict the dependent variable $\bx$ at the next time step. Given a stable linear system \eqref{eq:LinearDynamics} and the state $\bx$ at time $t_i$, denoted by $\bx(t_i)$, we can estimate $\bx(t_{i+1})$ as follows:
\begin{equation}\label{eq:integration}
\bx(t_{i+1}) = \Phi_{\dt}\left(\bA, \bx(t_i)\right) := \bx(t_i) + \dfrac{\dt}{6}\left(\bh_1 +2\bh_2 + 2\bh_3 + \bh_4\right),
\end{equation}
where
\begin{equation}
\begin{aligned}
\bh_1 &= \bA\bx(t_i), & \bh_2 &= \bA\left(\bx(t_i) + \dfrac{\dt}{2}\bh_1\right), &
\bh_3 &= \bA\left(\bx(t_i) + \dfrac{\dt}{2}\bh_2\right), & \bh_4 &= \bA\left(\bx(t_i) + \bh_3\right).
\end{aligned}
\end{equation}
We can then modify the objective function \eqref{eq:Loss_function} as follows:
\begin{equation}\label{eq:Loss_function_RK}
\min_{\bar\bJ,\, \bar\bR, \bar\bQ} \sum_i\left\|{\bx}(t_{i+1})- \Phi_{\texttt{\dt}}\left(\bar\bA, \bx(t_i)\right)\right\|, \quad \text{with}~~\bar\bA = \left(\bar\bJ - \bar\bJ^{\top} - \bar\bR\bar\bR^\top\right) \bar\bQ\bar\bQ^\top,
\end{equation}
thus avoiding the necessity of derivative information.  Consequently, we have the optimization problem \eqref{eq:Loss_function_RK} whose solutions are stable by construction. This proposed methodology is called \emph{stable linear system inference} (\newmethod) throughout this paper.   

Moreover, we highlight that one can use any other numerical integration or higher-order schemes, as well as multi-step unfolding schemes. In fact, the concept of adjoint sensitivity, as proposed in neural ODEs \cite{chen2018neural}, can be employed to perform efficient computation (in term of memory cost) through any numerical integrator, but it often comes with more CPU cost. 


\section{Further Considerations}\label{sec:furtherconsider}
In this section, we discuss a few possible considerations and extensions for \newmethod.
\subsection{Low-dimensional representation and analog to compressed DMD}\label{sec:compression}
Collecting data from physical systems often results in high-dimensional data with thousands to millions of degrees of freedom. As a result, building models from such high-dimensional data are computationally expensive. However, mostly these high-dimensional data reside in a low-dimensional subspace, so we can obtain a compact, low-dimensional representation by projecting onto such a subspace. Techniques such as principal component analysis (PCA), proper-orthogonal decomposition (POD), and autoencoders can be used to achieve this goal. Here, we describe the POD approach, which is commonly used in the reduced-order modeling community.

Singular-value decomposition (SVD) is a key element of POD and allows us to determine a suitable low-dimensional subspace. To that aim, let us consider a data matrix $\bX \in \R^{n\times N}$, where $n$ is the dimension of the data, and $N$ is the number of measurements. We can determine the dominant subspace by performing the SVD of $\bX$, i.e.,
\begin{equation}
	\bX := \bU\Sigma \bV^{\top},
\end{equation}
where the dominating subspaces are contained in $\bU$, and the singular values stored in $\Sigma$ determine their significance.  $\Sigma := \diag{\sigma_1,\ldots,\sigma_n}$ is a diagonal matrix with $\sigma_{i+1}\geq \sigma_i$.  By using the subspace spanned by the $r$ most dominant left singular vectors, denoted by $\bU_r$, we can obtain a low-dimensional representation as follows:
\begin{equation}
	\bX_r = \bU_r^{\top}\bX,
\end{equation}
where $\bX_r \in \R^{r\times N}$. Moreover, we can reconstruct our high-dimensional data $\bX$, indicated by $\tilde\bX$, using  $\bX_r$, e.g., $\tilde\bX = \bU_r\bX_r$. Additionally, the information loss due to projection and re-projection can be measured as follows:
\begin{equation}
	\|\bX - \tilde\bX\|_2 \leq  \sum_{r+1}^n\sigma_i.
\end{equation}
Using the above relation as an indicator, we can choose a suitable $r$. Finally, once we have a low-dimensional representation of the data $\bX_r$, we can use the method \newmethod~described in the previous sections to learn linear dynamical systems in the low-dimensional space, thus avoiding numerical computations in high dimensions, which may be too expensive. 

\subsection{Systems controlled by external inputs} 
Many dynamic processes are controlled by external inputs, which can result in different dynamic behavior when varied. To account for this in our modeling, we consider the following model hypothesis:
\begin{equation}\label{eq:linear_control}
	\dot{\bx}(t) =  \bA\bx(t) + \bB\bu(t).
\end{equation}
Given the trajectories $\bx(t)$ subject to various control inputs, we aim to learn \eqref{eq:linear_control} while maintaining stability constraints. To this aim, we collect snapshots of the input functions as  follows
\begin{equation}\label{ea:input_data}
	\bU = \begin{bmatrix} 
		| & | & & |
		\\ 
		\bu(t_0) & \bu(t_1)& \ldots & \bu(t_N) 
		\\
		| & | & & |
	\end{bmatrix} \in \R^{1 \times N},
\end{equation}
and tailor the objective function in \eqref{eq:Loss_function} to include the control input. Hence, the resulting objective for inferring control systems is given by
\begin{equation}\label{eq:Loss_function_control}
	(\bar\bJ, \bar\bR, \bar\bQ, \bar\bB)= \underset{\acute\bJ, \acute\bQ, \acute\bR, \acute\bB}{\arg\min} \|\dot{\bX}-(\acute\bJ -\acute\bJ^{\top} - \acute\bR\acute\bR^\top) \acute\bQ\acute\bQ^\top\bX  \acute\bB\bU\|_F,
\end{equation}
and the learned dynamical system with control is guaranteed to be stable.

 With some slight modifications, we can easily extend the discussion from Subsection~\ref{subsec:stability_parameterization} and incorporate numerical schemes to avoid the derivative data. Specifically, we modify \eqref{eq:integration} as follows:
\begin{equation}\label{eq:integration_control}
	\bx(t_{i+1}) \approx \Phi^{\texttt{c}}_{\texttt{\dt}}\left(\bA, \bB, \bx(t_i),\bu\right) := \bx(t_i) + \dfrac{\dt}{6}\left(\bh_1 +2\cdot\bh_2 + 2\cdot\bh_3 + \bh_4\right), \qquad \text{with}~~\dt = t_{i+1} - t_i
\end{equation}
\begin{equation}
	\begin{aligned}
		\bh_1 &= \bA\bx(t_i) + \bB\bu(t_i), & \bh_2 &= \bA\left(\bx(t_i) + \dfrac{\dt}{2}\bh_1\right) + \bB\bu\left(t_i + \tfrac{\dt}{2}\right) \\
		\bh_3 &= \bA\left(\bx(t_i) +\dfrac{\dt}{2}\bh_2\right) +  \bB\bu\left(t_i + \tfrac{\dt}{2}\right), ~\text{and} & \bh_4 &= \bA\left(\bx(t_i) + \bh_3\right) +  \bB\bu(t_i + dt).
	\end{aligned}
\end{equation}
The objective function \eqref{eq:Loss_function_control} then changes as follows:
\begin{equation}\label{eq:Loss_function_RKcontrol}
	\min_{\bar\bJ,\, \bar\bR, \bar\bQ,\bar\bB} \sum_i\left\|{\bx}(t_{i+1})- \Phi^{\texttt{c}}_{\texttt{\dt}}\left(\bar\bA, \bar\bB,\bx(t_i)\right)\right\|_F, \quad \text{with}~~\bar\bA = \left(\bar\bJ - \bar\bJ^{\top} - \bar\bR\bar\bR^\top\right) \bar\bQ\bar\bQ^\top.
\end{equation}
This allows us to obtain a stable control system realization by construction and avoid the need to compute derivative information by combining it with a numerical integrator.

\subsection{Stable  linear models for predefined observers} 
It is well-known that nonlinear dynamic systems can be represented as linear systems in an infinite-dimensional Hilbert space, as described in \cite{koopman1931hamiltonian}. However, in order to use these models for engineering purposes, it is necessary to approximate this infinite-dimensional space. The extended DMD method proposed in \cite{williams2015data} aims to design observers that can accurately approximate infinite-dimensional operators using finite-dimensional linear operators. However, the methodology to learn stable models is maintained for newly designed observers. Hence, the method for learning stable realizations described in the previous section (\newmethod) can be readily applied in the context of extended DMD to learn stable linear operators for continuous-time systems.
%

\section{Numerical Experiments}\label{sec:Exp}
In this section, we evaluate the performance of \newmethod~through three numerical examples. 
We compare \newmethod~with the recently proposed methodology in \cite{uy2022operator} which allows inference for LDS using numerical integration un-rolling to avoid the need for derivative computation but without any enforcing of stability of the learned operator. We refer to it as \emph{linear system inference} (\benchmark).
 All experiments are conducted using PyTorch, with 20,000 updates performed using the Adam optimizer \cite{kingma2014adam} and a triangular cyclic learning rate ranging from $10^{-6}$ to $10^{-2}$. The coefficients of all matrices are randomly generated from a Gaussian distribution with a mean of $0$ and standard deviation of $0.1$.

%
\subsection{Unsteady flow for cylinder-wake example}
In this example, we study the flow past a cylinder, a widely used benchmark problem in the literature (see, for instance, \cite{morKutBBetal16}). The Reynolds number is set to $\texttt{Re}= 100$ and we have $151$ vorticity measurements, which exhibits periodic oscillations. The data was collected on a grid of $199\times 449$. Since the data can be represented very well in a low-dimensional manifold, we compress the data using $31$ dominant POD modes, consequently reducing the complexity of the optimization problem and so for the learned model. We learn continuous-time linear models using \newmethod~and \benchmark. To assess the performance, we first compare the eigenvalues of these learned systems in \Cref{fig:clyinder_eigenvalues}, which clearly indicates 
\newmethod~yields a model with all eigenvalues in the left-half plane. In contrast, \benchmark~yields a model, possessing some unstable eigenvalues. Furthermore, we compute the inferred eigenfunctions from the model obtained with \newmethod. In  \Cref{fig:clyinder_eigenfuncs}, we plot the real part of the four dominant eigenfunctions  showing the dominant flow dynamics as expected.
\begin{figure}[!tb]
	\centering
	\includegraphics[width = 0.75\textwidth]{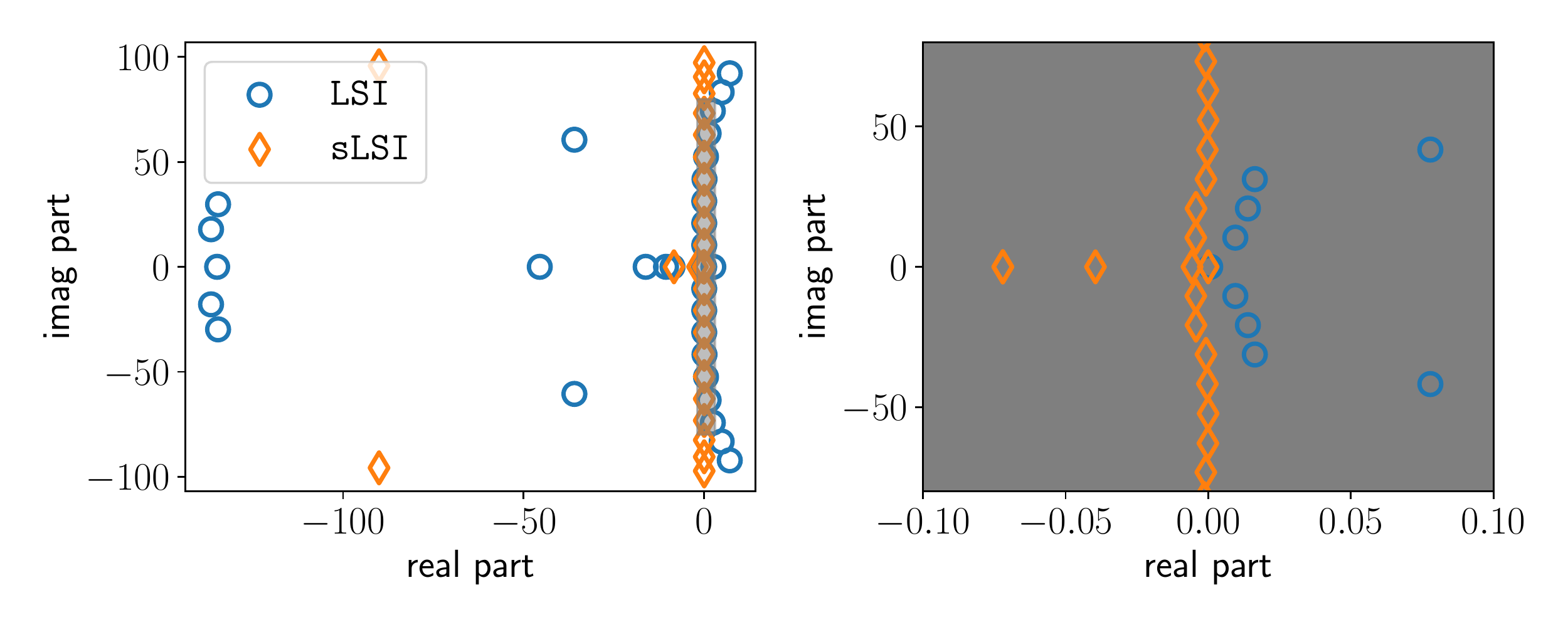}
	\caption{Flow-past cylinder example: A comparison of the eigenvalues of the learned models, which clearly shows 
		\newmethod~guarantees that all of the eigenvalues are in the left half-plane. The right plot shows a zoomed-in version of the left plot centered around the origin.}
	\label{fig:clyinder_eigenvalues}
\end{figure}
\begin{figure}[!tb]
	\centering
	\includegraphics[width = 0.35\textwidth]{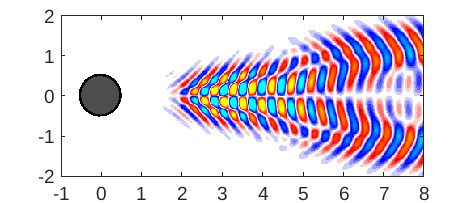}
	\includegraphics[width = 0.35\textwidth]{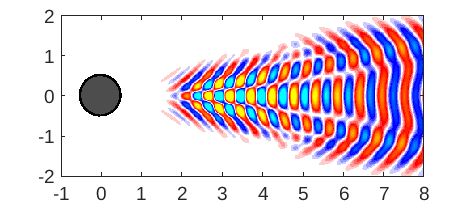}\\
	\includegraphics[width = 0.35\textwidth]{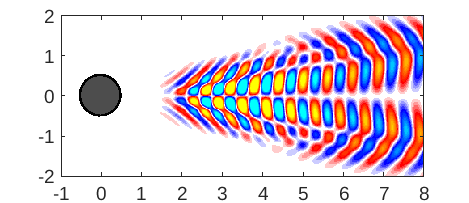}
	\includegraphics[width = 0.35\textwidth]{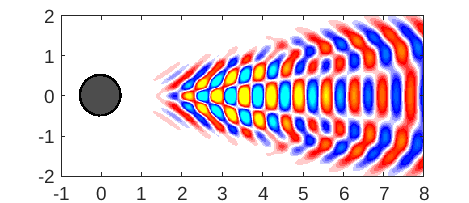}
	\caption{Flow-past cylinder example: The plots show the real part of the dominant eigenfunctions from the learned model via \newmethod.}
	\label{fig:clyinder_eigenfuncs}
\end{figure}

\subsection{Transporting flow}
\begin{wrapfigure}{o}{0.4\textwidth}
	\vspace{-0.25cm}
	
	\includegraphics[width =0.4\textwidth]{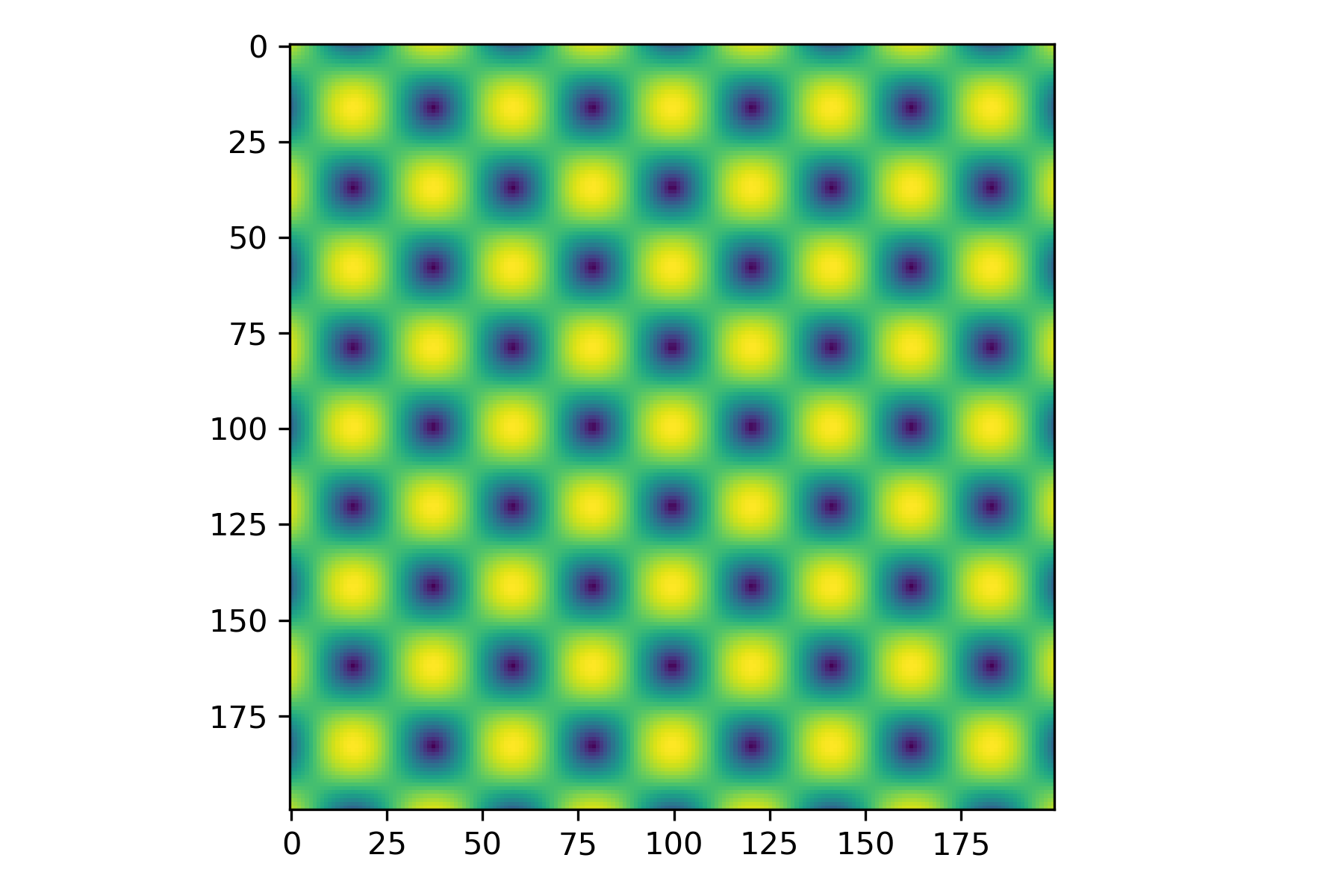}
	\vspace{-0.5cm}

	\caption{Transporting flow: The plot shows the magnitude of the velocity at time $t = 0$.}
	\label{fig:transportflow_data}

\end{wrapfigure}
Next, we consider a two-dimensional transporting flow whose $x$ and $y$-directions velocities, denoted by $u$ and $v$, are given by 
\begin{equation}
	\begin{aligned}
		u(x,y,t) &= \sin(5(t-x))\sin(5(t-y)),\\
		v(x,y,t) &= \cos(5(t-x))\cos(5(t-y)).
	\end{aligned}
\end{equation}

We consider the spatial domain $[-1.5,1.5]\times[-1.5,1.5]$ and $200$ equidistant points in both directions to sample the flow. Then, we collect $100$ data points in time interval $[0,5]$, and the flow at time $t=0$ is shown in \Cref{fig:transportflow_data}. 

The data in this study has a high dimensionality, with a total of $2\cdot 40,000$ variables. To address this, we first aim to find a low-dimensional representation of the data using POD. We discover that the data can be described by only three dominant modes (up to  the machine precision). Hence, we project the high-dimensional data onto these dominant subspaces and use them to learn linear dynamical models with \newmethod~and \benchmark. The eigenvalues of these models are depicted in \Cref{fig:transport_eigenvalues}. The figure clearly demonstrates that the proposed methods ensure stability. In contrast, when stability is not imposed, then the resulting systems can be unstable, which is the case for this example. Moreover, we show the eigenfunctions related to the velocities in the $x$ and $y$-directions in \Cref{fig:transport_eigenfuncs}, showing transportion of the flow in a specific direction. 
\begin{figure}[tb]
	\centering
	\includegraphics[width = 0.7\textwidth]{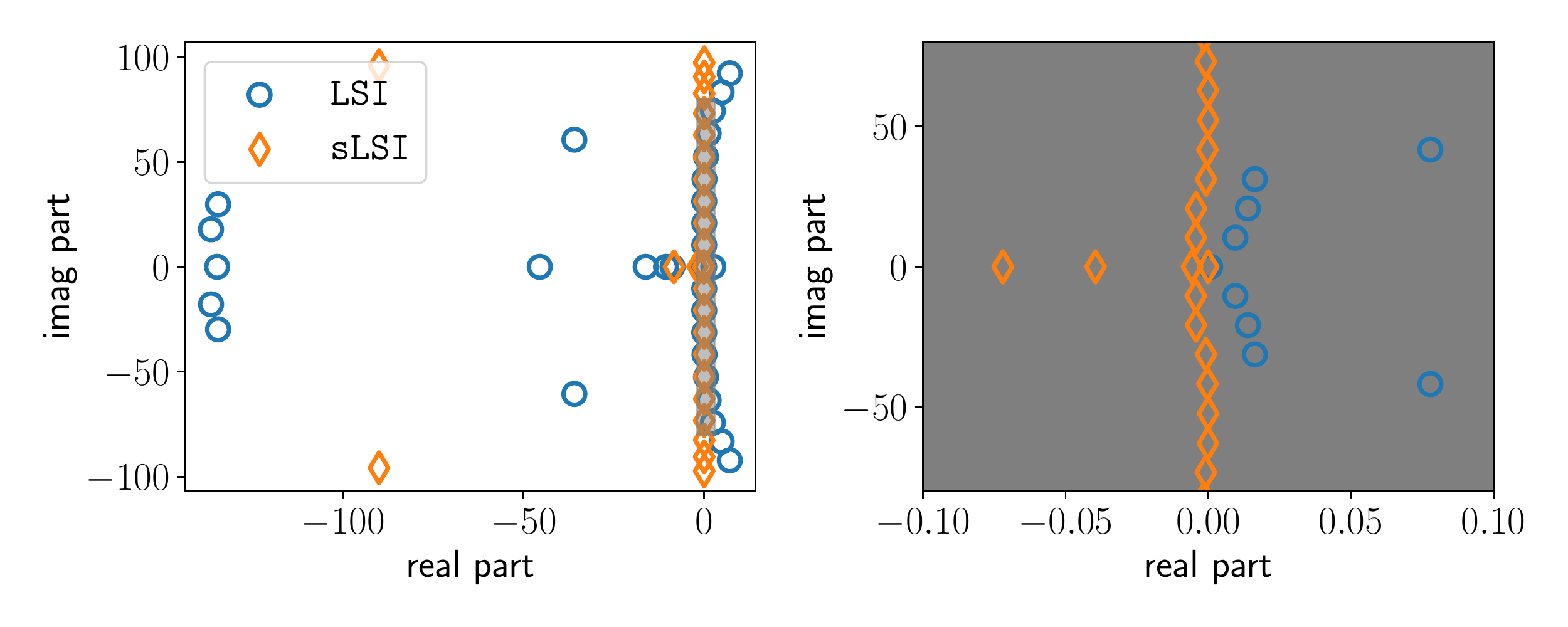}
	\caption{Transport flow example: A comparison of the eigenvalues of the learned models, which clearly shows \newmethod~guarantees to have all eigenvalues in the left half-plane. The right plot shows a zoomed-in version of the left plot centered around the origin.}
	\label{fig:transport_eigenvalues}
\end{figure}
\begin{figure}[!tb]
	\centering
	\includegraphics[width = 0.33\textwidth, trim= 0cm 0cm 19cm 0cm, clip]{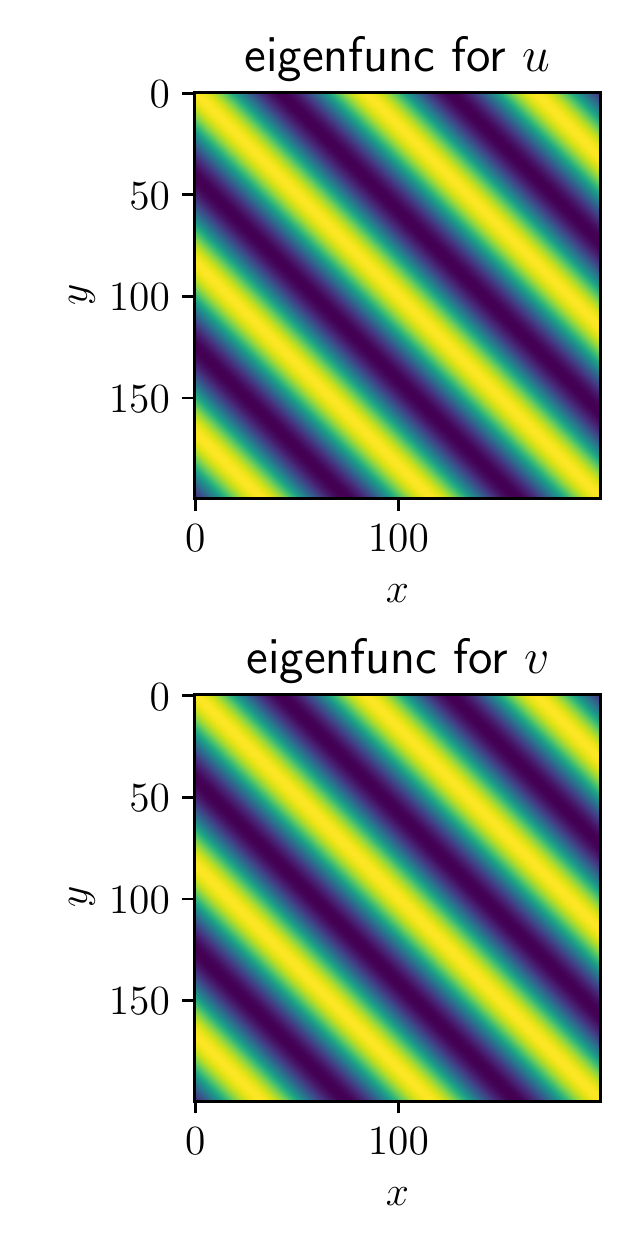}
	\includegraphics[width = 0.66\textwidth, trim= 12.5cm 0cm 0cm 0cm, clip]{\pathfig/Superposedflow/ImposeStability/eigenfunc.pdf}
	\caption{Transport flow example: The plots show the dominant eigenfunctions for $u$ and $v$, i.e., the velocities in the $x$ and $y$-directions, respectively.}
	\label{fig:transport_eigenfuncs}
\end{figure}

\subsection{Burgers' equations}
In our last example, we consider the one-dimensional Burgers' equation with the following initial and boundary:
\begin{equation}
	\begin{aligned}
		v_t +  v\dot v_\zeta&=  + \mu v_{\zeta\zeta} & & \text{in}~(0,1)\times (0,T), \\
		v_\zeta(0,\cdot)  & = 0, &&\\
		u_\zeta(1,\cdot) &= 0, &&\\
		u(\zeta,0) &= v_0(\zeta) & & \text{in}~(0,1),
	\end{aligned}
\end{equation}
where $v_t$ and $v_{\zeta}$ denote the derivative of $v$ with respect to time $t$ and space $\zeta$, and $v_{\zeta\zeta}$ denotes the second  derivative of $v$ with respect to $\zeta$. The PDE is discretized using $1,000$ equidistant grid points. We collect $500$ data points for a given initial condition in the time interval $[0,1]$s. Furthermore, we assume the initial conditions to be $v_0(\zeta) = 1 + \sin((2f\zeta+1)\pi)$, where $f = [1, 1.25,\ldots, 4.75,5.0]$. Hence, we have the data corresponding to in-total $17$ different initial conditions. Next, we split the data into training and testing, and for the testing, we take the initial conditions with $f = [1.75, 2.75, 3.75]$, and the rest are considered for learning. 

First of all, since the data are high-dimensional, we seek to determine a low-dimensional representation of the training data by projecting it using the dominant POD basis. We consider $21$ modes, which capture more than $99.9\%$ of the total energy. Next, we employ \newmethod~and \benchmark~to obtain linear models. For comparison, we first look at the eigenvalues of both learned models, which is shown in \Cref{fig:burgers_eigenvalues}. It indicates the guaranteed stability of 
\newmethod, whereas \benchmark~has an eigenvalue in the right-half plane, indicating potential instability in long-term prediction.

Moreover, we compare the learned models on the left-out initial conditions for testing. For this, we first project the initial condition onto a $21$-dimensional subspace using the same POD basis as used for the training data. We integrate the system for the time-interval $[0,1]$ and take $500$ steps in this interval. For comparison, we examine the relative $L_2$-error for all three testing initial conditions, shown in \Cref{fig:burgers_test}. It shows that \newmethod~even performs well, up to a factor of two, in the time-domain simulations despite having stability constraints on the linear operator while inferring. We also compare the time-domain simulations for one of the test cases in \Cref{fig:burgers_test_time_simulations}. It is done by first integrating the low-dimensional model, which is followed by re-projecting onto the higher-dimensional using the POD basis. 

\begin{figure}[tb]
	\centering
	\includegraphics[width = 0.9\textwidth]{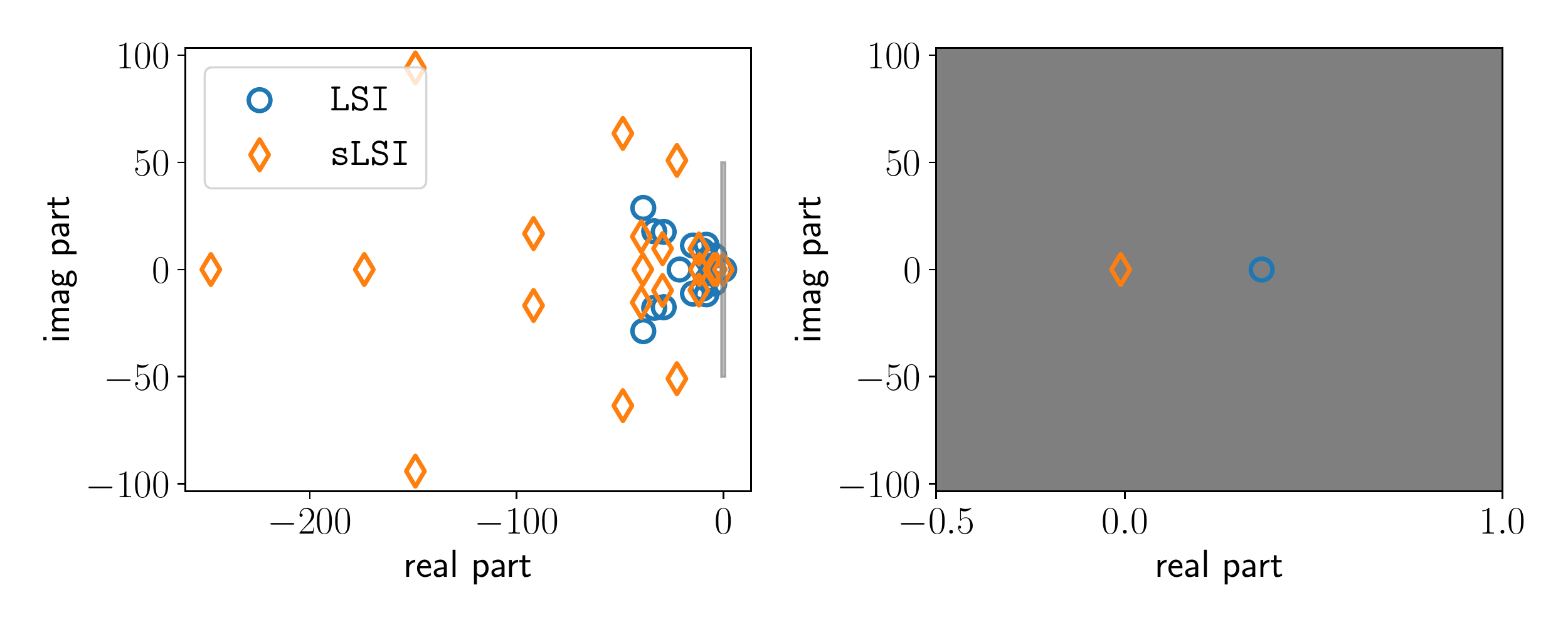}
	\caption{Burgers' equation: A comparison of the eigenvalues of the learned models, which clearly shows that \newmethod~guarantees not to have any eigenvalues in the right half-plane.}
	\label{fig:burgers_eigenvalues}
\end{figure}

\begin{figure}[tb]
	\centering
	\includegraphics[width = 0.5\textwidth]{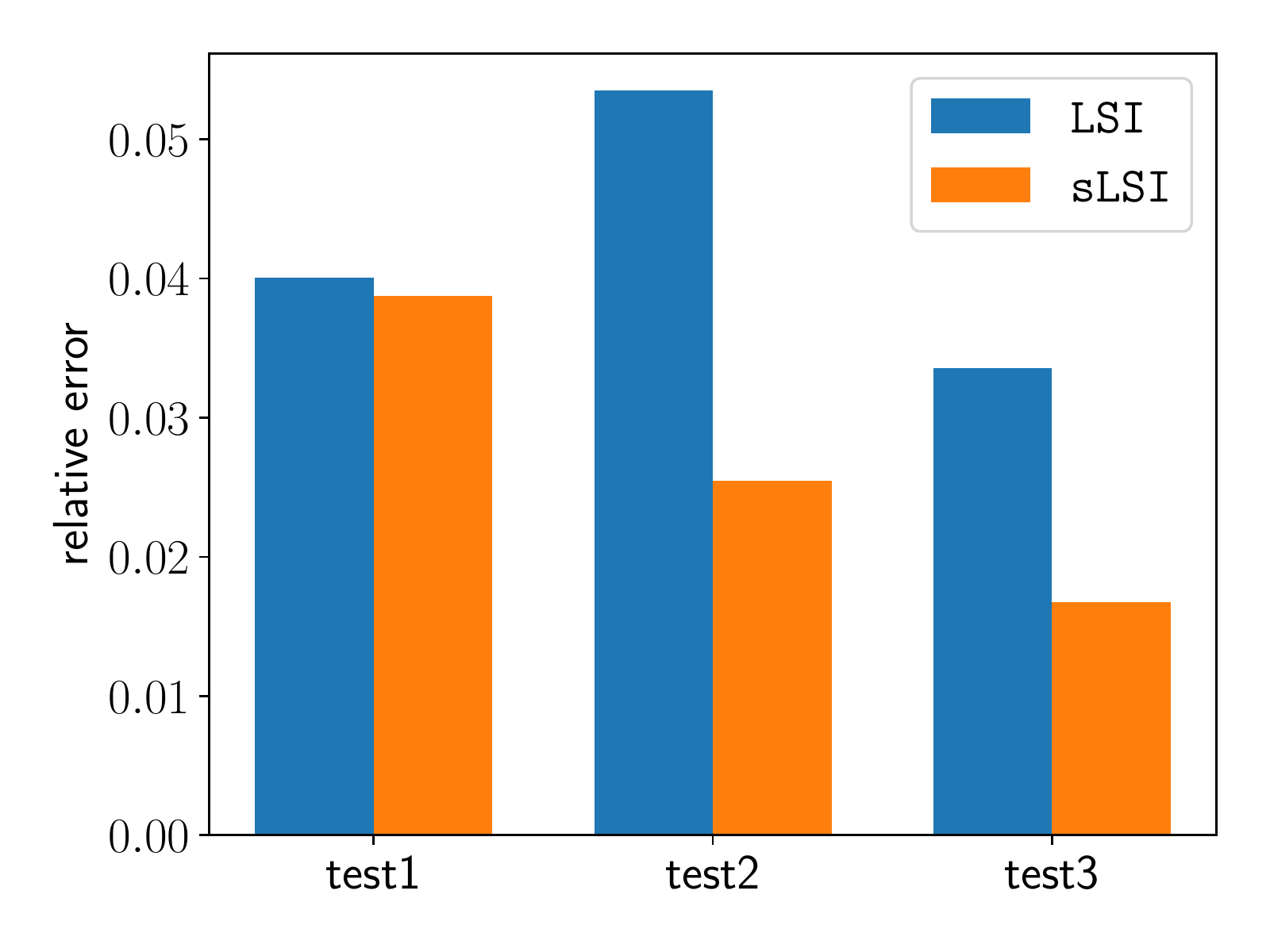}
	\caption{Burgers' equation: A performance of the learned models on the test data. }
	\label{fig:burgers_test}
\end{figure}

\begin{figure}[tb]
	\centering
	\includegraphics[width = 0.9\textwidth]{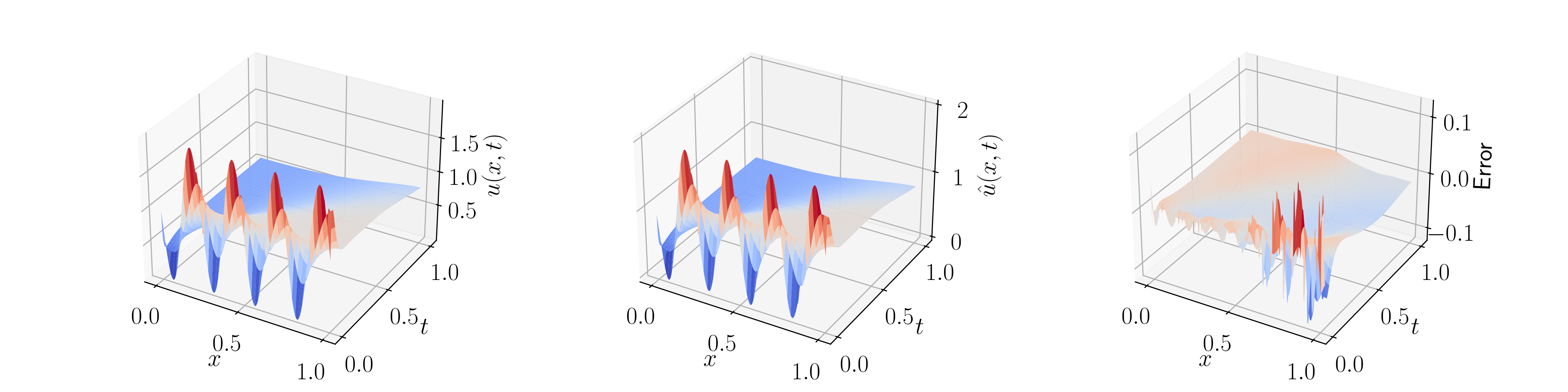}
	\caption{Burgers' equation: A comparison of the time-domain simulation of \newmethod~on a test data with $f = 3.75$. The left figure is the ground truth; the middle is \newmethod, and the right one shows the error between them. }
	\label{fig:burgers_test_time_simulations}
\end{figure}

\section{Conclusions}\label{sec:Conc}
We discussed a method for learning continuous-time  linear dynamical systems (\newmethod)~with two key features. First, the \newmethod~approach makes use of the parameterization of stable matrices, ensuring the stability of the inferred systems through direct encoding. The second feature of the proposed method is the use of an integral form for learning continuous-time systems, thus eliminating the requirement for explicit derivative computation. Several extensions are presented as well.  Numerical examples demonstrate the stability features of the learned models. In contrast, this characteristic can be broken if stable-matrix parameterization is not applied. Moreover, the performance of the learned models using \newmethod~can lead to an improvement of up to a factor of two in time-domain simulations.

This work opens up the possibility for further research. As emphasized, having an adequate parameterization, which directly encodes desired properties, guarantees to have those properties by construction. Therefore, it would be intriguing to explore the possibility of enforcing additional physical laws, such as conservation of mass and energy, through the use of appropriate parameterizations. An extension to parametric varying systems could also be promising.

\bibliographystyle{elsarticle-num} 
\bibliography{mor,igorBiblio}

\end{document}